\documentclass[letterpaper]{article} 
\usepackage{aaai25}  
\usepackage{times}  
\usepackage{helvet}  
\usepackage{courier}  
\usepackage[hyphens]{url}  
\usepackage{graphicx} 
\usepackage{natbib}  
\usepackage{caption} 
\usepackage{algorithm}
\usepackage{algorithmic}

\usepackage{newfloat}
\usepackage{listings}
\DeclareCaptionStyle{ruled}{labelfont=normalfont,labelsep=colon,strut=off} 
\floatstyle{ruled}
\newfloat{listing}{tb}{lst}{}
\floatname{listing}{Listing}
\title{Solving Epistemic Logic Programs using Generate-and-Test with Propagation}
\author{Jorge Fandinno\textsuperscript{\rm 1} and Lute Lillo\textsuperscript{\rm 1,2}}
\affiliations{
    \textsuperscript{\rm 1}University of Nebraska Omaha, Omaha, NE, USA\\
    jfandinno@unomaha.edu\\
    \textsuperscript{\rm 2}University of Vermont, Burlington, VT, USA\\
    eleuterio.lillo-portero@uvm.edu
}

\usepackage{amsmath}
\usepackage{amsthm}
\usepackage{amssymb}
\newtheorem{theorem}{Theorem}
\newtheorem*{mtheorem}{Main Theorem}
\newtheorem{lemma}{Lemma}
\newtheorem{proposition}{Proposition}
\newtheorem{definition}{Definition}
\newtheorem{corollary}{Corollary}
\newtheorem{example}{Example}

\newcommand{\K}{\ensuremath{\mathbf{K}}}

\newcommand{\nott}[1]{\ensuremath{\mathit{n#1}}}
\newcommand{\notnot}[1]{\ensuremath{\mathit{nn#1}}}
\newcommand{\kk}[1]{\ensuremath{\mathit{k#1}}}
\newcommand{\kknott}[1]{\ensuremath{\mathit{kn#1}}}
\newcommand{\kp}[1]{\ensuremath{\mathit{pk#1}}}
\newcommand{\kpnott}[1]{\ensuremath{\mathit{pkn#1}}}

\newcommand{\NF}[1]{\ensuremath{\mathit{NF}(#1)}}
\newcommand{\AT}{\ensuremath{\mathit{At}}}
\newcommand{\At}[1]{\ensuremath{\mathit{At}(#1)}}

\newcommand{\restr}[2]{{#1}_{\!|\!#2}}
\newcommand{\SM}[1]{\text{\rm SM}[#1]}
\newcommand{\SMC}[1]{\text{\rm CC}[#1]}

\newcommand\wv{\mathbb{W}}

\newcommand{\EPASP}{\texttt{EP-ASP}}
\newcommand{\eclingo}{\texttt{eclingo}}

\newcounter{programcounter}
\newcommand{\newprogramph}{\refstepcounter{programcounter}}

\newcommand{\program}[1]{\ensuremath{\Pi_{#1}}}
\newcommand{\speedup}[1]{\raisebox{0.1em}{\scriptsize$\sim$}#1}
\usepackage{pgfplots}
\begin{document}

\maketitle

\begin{abstract}
This paper introduces a general framework for generate-and-test-based solvers for epistemic logic programs that can be instantiated with different generator and tester programs.
It proves sufficient conditions on those programs for the correctness of the solvers built using this framework.
It also introduces a new generator program that incorporates the \emph{propagation of epistemic consequences} and shows that this can exponentially reduce the number of candidates that need to be tested while only incurring a linear overhead.
We implement a new solver based on these theoretical findings and experimentally show that it outperforms existing solvers by achieving a~\speedup{3.3x} speed\nobreakdash-up and solving 91\% more instances on well\nobreakdash-known benchmarks.
\end{abstract}


\section{Introduction}\label{sec:introduction}

Answer Set Programming (ASP) is a declarative programming language well\nobreakdash-suited for solving knowledge\nobreakdash-intensive search problems, which allows users to encode problems such that the resulting output of the program (called stable models or answer sets) directly corresponds to solutions of the original problem~\cite{gellif88b,gellif91a,lifschitz08b,breitr11a,schwol18a,lifschitz19a}.
Epistemic logic programs (ELPs) extend the ASP language by allowing the use of \emph{subjective literals} in the body of rules~\cite{gelfond91a,gelpri93a,gelfond94a,fafage22a}.
As an example, rules
\begin{align*}
\mathit{felon}   &\leftarrow \K \mathit{break\_rule}
\\
\mathit{suspect} &\leftarrow \neg \K \mathit{break\_rule}, \neg \K \neg \mathit{break\_rule}
\end{align*}
encode that a person is a~$\mathit{felon}$ if we can determine that she broke a rule, and she is a~$\mathit{suspect}$ if it cannot be determined that she has broken it nor that she has not done so~\cite{solekale17a}.
ELPs allow us to naturally represent problems that involve reasoning about the lack of knowledge of agents---as illustrated by the above example---as well as problems laying on the second level of the polynomial hierarchy such as conformant planning~\cite{solekale17a,kawabagezh20a, cafafa19b, cafafa21a}, action reversibility~\cite{famoch21a} or reasoning about attack trees and graphs~\cite{fafage22a}.
It is worth noting that the computational complexity of ELPs is higher than that of ASP: the problem of determining whether an ELP has a solution (called \emph{worldview}) is $\Sigma^{P}_{3}$\nobreakdash-complete~\cite{truszczynski11b}, one level higher in the polynomial hierarchy than the complexity deciding whether an objective program has a stable model~\cite{daeigovo01a}.
%
%
Hence, the development of efficient tools to solve ELPs, called ELP solvers, is crucial for exploiting the high expressivity capabilities of ELPs in practice.
Existing ELP solvers can be classified in three categories: \emph{post\nobreakdash-processing}-based, \emph{generate\nobreakdash-and\nobreakdash-test}-based, and \emph{translational}-based.
Early solvers were {post\nobreakdash-processing\nobreakdash-based} as they use an ASP solver to compute all stable models of some objective (or non\nobreakdash-epistemic) program that are post-processed to obtain its worldviews (see~\citealt{leckah18a} for a survey).
Generate\nobreakdash-and\nobreakdash-test\nobreakdash-based solvers on the other hand use two instances of an ASP solver---one to generate candidate worldviews and another to check whether they are indeed worldviews---avoiding the generation of all stable models in most practical cases~\cite{solekale17a,cafagarosc20a}.
Translational-based solvers translate ELPs into an alternative language that matches the computational complexity of ELPS such as non\nobreakdash-ground objective programs with large rules~\cite{bimowo20a} or ASP with Quantifiers~\cite{fabmor23a}.
Existing generate\nobreakdash-and\nobreakdash-test and translational based solvers outperform post\nobreakdash-processing based solvers in existing benchmarks~\cite{leckah18a,cafagarosc20a}, while the literature lacks clear evidence when comparing generate\nobreakdash-and\nobreakdash-test and translational based solvers.
Our own experiments show that existing generate\nobreakdash-and\nobreakdash-test-based solvers outperform existing translational-based solvers (See the Experimental Evaluation Section below).

In this paper, (i) we introduce a general framework for generate\nobreakdash-and\nobreakdash-test-based solvers that can be instantiated with different generator and tester programs, and we prove sufficient conditions on those programs for the correctness of the solvers built using this framework;
(ii) we instantiate this framework with a new generator program that incorporates the \emph{propagation of epistemic consequences} and prove that this can exponentially reduce the number of candidates that need to be tested while only incurring in a linear overhead; and (iii) we implement a new solver based on the two previous points and experimentally demonstrate that it can solve 91\% more instances than the best performing available solver with a \speedup{3.3x} speed\nobreakdash-up.

\section{Background}\label{sec:background}

We assume some familiarity with the answer set semantics for disjunctive logic programs~\cite{gellif91a}.
Given a set of atoms $\AT$, an \emph{objective literal} is either an atom or an atom preceded by one or two occurrences of the  negation symbol~``$\neg$.''
An \emph{extended objective literal} is either an objective literal or a truth constant%
\footnote{For a simpler description of program transformations, we allow truth constants in rules, where $\top$ denotes true and $\bot$ denotes false.  These constants can be easily removed.%
}.
%
%
An expression of the form $\K l$ with $l$ being an extended objective literal is called \emph{subjective atom}.
A \emph{subjective literal}~$L$ is a subjective atom possibly preceded by one or two occurrences of the negation symbol.
A \emph{literal} is either an extended objective literal or a subjective literal.
A \emph{rule} is an expression of the~form:
\begin{gather}
a_1 \vee \dots \vee a_n \ \leftarrow \ L_1, \dots, L_m
	\label{eq:rule}
\end{gather}
with $n\geq 0$ and $m\geq 0$, where each ${a_i \in \AT}$ is an atom and each $L_j$ a literal.
The left- and right-hand sides of~\eqref{eq:rule} are respectively called the \emph{head} and \emph{body} of the rule.
When~${n=0}$, the rule is called a \emph{constraint} and we usually write~$
\bot$ as its head.
A \emph{choice rule} is an expression of the form:
\begin{gather}
\{ a \} \leftarrow L_1, \dots, L_m
    \label{eq:choice}
\end{gather}
and it is understood as a shorthand for the rule:
\begin{gather}
a \leftarrow L_1, \dots, L_m, \neg\neg a
    \label{eq:choice2}
\end{gather}
An \emph{epistemic program} (or epistemic specification) is a finite set of rules.
Given an epistemic program $\Pi$, we define $\At{\Pi}$ as the set of all atoms that occur in program~$\Pi$.
An epistemic program without subjective atoms is called an~\emph{objective program}.
Similarly, a rule without subjective atoms is called \emph{objective rule}.
In many descriptions of epistemic logic programs, a second epistemic operator~$\mathbf{M}$ is also introduced.
Though interesting from a knowledge representation perspective, this operator is not necessary from the solving perspective because it can be replaced by~$\neg\K\neg$~\cite{fafage22a}.
Therefore, we focus on programs with the operator~$\K$ only.

An interpretation~${I \subseteq \AT}$ is a set of atoms.
We say that an interpretation~$I$ satisfies an extended objective literal~$L$ if~$L$ is~$\top$, or $L$ is an atom and~${L \in I}$, or~$L$ is of the form~$\neg a$ and~${a \not\in I}$, or~$L$ is of the form~$\neg\neg a$ and~${a \in I}$.
An interpretation~$I$ \emph{satisfies} an objective rule if it satisfies some atom in the head of the rule whenever it satisfies all the literals in its body.
An interpretation~$I$ is a \emph{model} of an objective program if it satisfies all rules of the program.
We say that a literal is \emph{negated} if it is of the form~$\neg A$ or~$\neg\neg A$ for some (subjective) atom~$A$.
The reduct of an objective program~$\Pi$ with respect to an interpretation~$I$, in symbols~$\Pi^I$, is obtained by removing all rules with a body that contains a negated literal~$L$ that is not satisfied by~$I$, and by removing all negated literals from the remaining rules.
An interpretation~$I$ is a \emph{stable model} of an objective program~$\Pi$ if~$I$ is a $\subseteq$\nobreakdash-minimal model of the reduct~$\Pi^I$.
$\SM{\Pi}$ denotes the stable models of~$\Pi$.

An \emph{assumption} is a set of objective literals~$A$ of the form~$a$ or~$\neg a$ such that it does not contain both~$a$ and~$\neg a$ for any atom~$a$.
A set of atoms~$M$ is a \emph{stable model} of an objective program~$\Pi$ \emph{under assumption~$A$} if~$M$ is a stable model of the program
\begin{gather*}
    \Pi \cup \{ \bot \leftarrow\neg a \mid a \in A \} \cup \{ \bot \leftarrow a  \mid \neg a \in A \}.
\end{gather*}
By~$\SM{\Pi;A}$ we denote the stable models of~$\Pi$ under the assumption~$A$.
By~$\SMC{\Pi;A}$ we denote the set of of \emph{cautions consequences} of~$\Pi$ \emph{under the assumption~$A$}, that is,
\begin{gather*}
    \SMC{\Pi;A} \quad = \quad \bigcap \SM{\Pi;A}
\end{gather*}

Let $\wv$ be a set of interpretations.
We write $\wv \models \K l$ if objective literal $l$ is satisfied by every interpretation~$I$ of $\wv$, and $\wv \models \neg \K l$ otherwise.

\begin{definition}[Subjective reduct; \protect\citealt{fafage22a}]
The \emph{subjective reduct} of an epistemic program~$\Pi$ with respect to a set of  interpretations~$\wv$, written~$\Pi^\wv$, is
obtained by replacing each subjective literal $L$ by $\top$ if $\wv \models L$ and by $\bot$ otherwise.\qed
\end{definition}

Note that the subjective reduct of an epistemic program does not contain subjective literals and, thus, it is an objective program.
Therefore, we can collect its stable models.

\begin{definition}
A \emph{belief interpretation}~$\wv$ is a non\nobreakdash-empty set of interpretations.
A \emph{worldview} of an epistemic program $\Pi$ is a belief interpretation~$\wv$ such that~$\wv=\SM{\Pi^\wv}$.    
\end{definition}

This definition of worldview is a rephrasing of the original definition by~\citet{gelfond94a}, usually denoted G94.
In this work, we focus on obtaining worldviews of epistemic programs under this semantics.
From a solving viewpoint, this is not a limitation because other semantics can use G94 as their basis.
There are simple linear\nobreakdash-time reductions from the G11~\cite{gelfond11a} and the K15~\cite{kawabagezh15a} semantics to G94~\cite{fafage22a}, while the worldviews according to the S16~\cite{sheeit16a,sheeit17a} and the C19~\cite{cafafa19a,cafafa20a} semantics are a selection of the~K15~\cite{kaleso16a,solekale17a} and the~G94 worldviews, respectively.

\section{Normal Form}\label{sec:normal.form}

To develop algorithms for solving epistemic logic programs, it is interesting to have a normal form that simplifies the number of cases the algorithm must consider.
This also significantly simplifies the presentation of the formal results.
We say that an epistemic program is in \emph{normal form} if negation does not occur in the scope of the~$\K$ operator, that is, every subjective literal is of the forms~$\K a$, $\neg\K a$ or~$\neg\neg\K a$ with~$a$ an atom.
%
%
We can transform any epistemic program~$\Pi$ into a corresponding program in normal form by 
\begin{itemize}
    
    \item replacing every subjective literal~$\K \neg a$ by~$\K \nott{a}$ and adding rule~$\nott{a} \leftarrow \neg a$ to~$\Pi$;
    
    \item replacing every subjective literal~$\K \neg\neg a$ by~$\K \notnot{a}$ and adding rule~$\notnot{a} \leftarrow \neg\neg a$ to~$\Pi$;
    
\end{itemize}
We assume that for every atom~$a$ in~$\Pi$, atoms~$\nott{a}$ and~$\notnot{a}$ do not occur in~$\Pi$.
%
%
By~$\NF{\Pi}$ we denote the program obtained from~$\Pi$ by applying these two steps.
In the following result~$\restr{\wv}{\At{\Pi}}$ stands for the set of interpretations obtained from~$\wv$ by removing all atoms that do not occur in~$\Pi$ from all the interpretations in~$\wv$.

\begin{theorem}\label{thm:normal.form}
    %
    There is a one\nobreakdash-to\nobreakdash-one correspondence between the worldviews of~$\Pi$ and~$\NF{\Pi}$ s.t.~$\wv$ is a worldview of~$\NF{\Pi}$ if and only if~$\restr{\wv}{\At{\Pi}}$ is a worldview of~$\Pi$.
\end{theorem}



\section{Worldviews by Generate-and-Testing }\label{sec:guess-and-check}

%
As mentioned in the introduction, generate\nobreakdash-and\nobreakdash-test\nobreakdash-based solvers rely on using two objective programs: a generator program that provides \emph{candidates} and a tester program that checks whether a candidate is a worldview.
Algorithm~\ref{alg:guess-and-check} summarizes how to compute~$n$ worldviews using this approach, where~$G(\Pi)$ and~$T(\Pi)$ respectively are a \emph{generator} and a \emph{tester program} for~$\Pi$.
\begin{algorithm}
    \caption{Generate-and-test computation of~$n$ worldviews of a program~$\Pi$ in normal form.}\label{alg:guess-and-check}
    \hspace*{\algorithmicindent} \textbf{Input} Generate program $G(\Pi)$\\
    \hspace*{\algorithmicindent} \textbf{Input} Test program $T(\Pi)$\\
    \hspace*{\algorithmicindent} \textbf{Input} Number of requested worldviews~$n$\\
    \hspace*{\algorithmicindent} \textbf{Output} Set~$\Omega$ containing at most~$n$  worldviews of~$\Pi$

    \begin{algorithmic}[1] 
        \STATE Let $\Omega=\emptyset$.
        \FOR{$M$ in $\SM{G(\Pi)}$} \label{alg.generator.loop.ini}
            \IF {\texttt{Test$(T(\Pi),\,M)$}}
            \STATE $\wv$ = \texttt{BuildWorldView$(M)$}
            \STATE $\Omega$ = $\Omega \cup \{ \wv \}$
                \IF {$|\Omega| \geq n$}
                \STATE \textbf{return} $\Omega$
                \ENDIF
            \ENDIF
        \ENDFOR \label{alg.generator.loop.end}
        \STATE \textbf{return} $\Omega$
    \end{algorithmic}
\end{algorithm}
Different generator and tester programs can be used to compute the worldviews of a program.
Here we provide general definitions for these programs with sufficient conditions to ensure that Algorithm~\ref{alg:guess-and-check} correctly computes the worldviews of a program.

We assume that, for each subjective literal~$\K a$ in~$\Pi$, there is a new fresh objective atom~$\kk{a}$ that does not occur anywhere else in the program.
For simplicity, we assume that no atom occurring in the program starts with the letter~$k$ and, thus, all atoms that start with this letter are newly introduced.
We denote by~$K(\Pi)$ the set of these newly introduced atoms.
Given an interpretation~${M \subseteq \At{\Pi} \cup K(\Pi)}$, we define~${k(M) = M \cap K(\Pi)}$ as the set of atoms of the form~${\kk{a} \in K(\Pi)}$ that belong to~$M$ with~${a \in \At{\Pi}}$.
For a worldview~$\wv$, by~$k(\wv)$ we denote the set of atoms of~${\kk{a} \in K(\Pi)}$ such that~$\wv \models \K a$.
In our characterization, each worldview~$\wv$ of a program~$\Pi$ is associated with a stable model~$M$ of an objective program such that~${k(\wv) = k(M)}$ satisfying some extra conditions that we introduce below.
Using this terminology, we can define our general notions of generator and tester program as follows.
%
%

\begin{definition}[Generator Program]\label{def:generator.program}
We say that an objective program~$G(\Pi)$ is a \emph{generator program} for a program~$\Pi$ if it satisfies the following conditions:
\begin{itemize}
    \item for every worldview~$\wv$ of~$\Pi$, there is~$M \in \SM{G(\Pi)}$ satisfying~${k(\wv) = k(M)}$ and~${(M \cap \At{\Pi}) \in \wv}$.
    
    \item every stable model~${M \in \SM{G(\Pi)}}$ and every set of interpretations of~$\wv$  satisfying~${k(\wv) = k(M)}$ also satisfy~${(M \cap\At{\Pi}) \in \SM{\Pi^\wv}}$.
\end{itemize}
\end{definition}
\noindent
The first condition in Definition~\ref{def:generator.program} ensures that, for every worldview of the original program, the generator program has a stable model that represents that worldview.
The second condition ensures that such a stable model corresponds to a non\nobreakdash-empty set, that is, a belief interpretation.
Recall, that empty sets of interpretations cannot be worldviews.

Before providing our general notion of a tester program, let us introduce an example of a tester program.
The intuitive idea of a tester program is that it behaves as the subjective reduct of the original program modulo the auxiliary atoms when used in combination with the proper assumptions.
For any rule~$r$ of the form of~\eqref{eq:rule}, by~$k(r)$ we denote the result of replacing each subjective literal~$\K a$ by~$\kk{a}$.
For any program~$\Pi$, by~$k(\Pi)$ we denote the program obtained from~$\Pi$ by replacing each rule~$r$ of~$\Pi$ by~$k(r)$.
Clearly, $k(\Pi)$ is an objective program over~$\At{\Pi} \cup K(\Pi)$.
By~$T_0(\Pi)$ we denote the program obtained from~$k(\Pi)$ by adding a choice rule~$\{ \kk{a} \}$ for each atom~$\kk{a} \in K(\Pi)$.
The following result shows that~$T_0(\Pi)$ behaves as the subjective reduct of~$\Pi$ modulo the auxiliary atoms under the proper assumptions.

\begin{proposition}\label{prop:basic.tester.program}
    Let~$\Pi$ be a program and~$\wv$ be a set of interpretations.
    Then, 
    \begin{gather}
        \{ M \cup k(\wv) \mid M \in \SM{\Pi^\wv} \} \ = \ \SM{T_0(\Pi); k(\wv)}
        \label{eq:1:prop:basic.generator.tester.program}
\\
    \SM{\Pi^\wv} \ = \ \restr{\SM{T_0(\Pi); k(\wv)}}{\At{\Pi}}
    \label{eq:1:cor:basic.generator.tester.program}
\end{gather}
\end{proposition}

\noindent
The following definition generalizes the above example of a tester program.

\begin{definition}[Tester Program]\label{def:tester.program}
We say that an objective program~$T(\Pi)$ is a \emph{tester program} for a program~$\Pi$ if it safisfies the following condition:
\begin{itemize}
    \item $\wv = \restr{\SM{T(\Pi); k(\wv)}}{\At{\Pi}}$ holds for every worldview~$\wv$ of~$\Pi$, and
    \item every non\nobreakdash-empty set of interpretations~$\wv$ that is not a worldview of~$\Pi$ satisfies~$\wv \neq \restr{\SM{T(\Pi); k(\wv)}}{\At{\Pi}}$
\end{itemize}
\end{definition}

\noindent
Using the second equality of Proposition~\ref{prop:basic.tester.program}, it is easy to see that program~$T_0(\Pi)$ is a tester program for~$\Pi$ because any non\nobreakdash-empty set of interpretations~$\wv$ is a worldview iff
\begin{gather*}
    \wv \ = \ \SM{\Pi^\wv} \ =\  \restr{\SM{T(\Pi); k(\wv)}}{\At{\Pi}}
\end{gather*}
Definition~\ref{def:tester.program} does only require that~$T(\Pi)$ behaves as the subjective reduct on actual worldviews.
For non\nobreakdash-worldviews, we only require that it does not produce false positives.
This relaxation may allow the introduction of further optimizations in the future.

\begin{definition}[Generate and test worldviews]
    If~$G(\Pi)$ and~$T(\Pi)$ respectively are a generator and a tester program for~$\Pi$, by~$\mathit{wv}(G(\Pi), T(\Pi))$ we denote the set of all stable models~$M$ of~$G(\Pi)$ such that
    \begin{gather}
        k(M) \quad = \quad \{ ka\in K(\Pi) \mid a \in \SMC{T(\Pi); k(M)} \}.
        \label{eq:1:thm:guess-and-check}
    \end{gather}
\end{definition}

\noindent
The following theorem shows the correspondence between the worldviews of~$\Pi$ and members of~$\mathit{wv}(G(\Pi), T(\Pi))$.

\begin{theorem}\label{thm:guess-and-check}
    Let~$G(\Pi)$ be any generator program for~$\Pi$ and~$T(\Pi)$ be any tester program for~$\Pi$.
    Then, there is a one\nobreakdash-to\nobreakdash-many correspondence between the worldviews of~$\Pi$ and the members of~$\mathit{wv}(G(\Pi), T(\Pi))$ such that
    \begin{enumerate}
        \item if $\wv$ is a worldview of~$\Pi$, then there is a stable model~$M$ of~$G(\Pi)$ with~$k(\wv) = k(M)$ satisfying~\eqref{eq:1:thm:guess-and-check};
        
        \item if~$M$ is a stable model of~$G(\Pi)$ satisfying~\eqref{eq:1:thm:guess-and-check}, then the set of stable models~$\SM{T(\Pi); k(M)}$ is a worldview of~$\Pi$.
    \end{enumerate}  
\end{theorem}

\noindent
Theorem~\ref{thm:guess-and-check} shows that, if we have any pair of generator and tester programs, then we can compute the worldviews using Algorithm~\ref{alg:guess-and-check} where the fuction~\texttt{Test}$(T(\Pi),\, M)$ amounts to check conditon~\eqref{eq:1:thm:guess-and-check}, and function~\texttt{BuildWorldView$(M)$} amounts to compute the stable models of the tester program under assumptions~$k(M)$.
When compared with existing generate\nobreakdash-and\nobreakdash-test-based solvers such as \EPASP~\cite{solekale17a} and \eclingo~\cite{cafagarosc20a}, function~\texttt{Test}$(T(\Pi),\, M)$ is simpler as it does not require an extra call for the computation of the brave consequences of the tester program.
This is thanks to computing the normal form before invoking Algorithm~\ref{alg:guess-and-check}%
\footnote{\citet{fabwol11a} describe an alternative way of checking when a candidate is a worldview using manifold programs. The manifold program is quadratic in size and we are not aware of any tool implementing this idea.}.

\section{Basic Generator Programs}\label{sec:reducing.candidates}

In this section, we introduce two basic generator programs that can be used to compute the worldviews of a program.
These two programs are inspired by the generator program used by~\EPASP\ and~\eclingo, respectively.

\paragraph{Tester program as a generator.}

We start by showing that the tester program~$T_0(\Pi)$ is also a generator program for~$\Pi$.

\begin{proposition}\label{prop:most.basic.generator.program}
    $T_0(\Pi)$ is a generator program for~$\Pi$.
\end{proposition}

Proposition~\ref{prop:most.basic.generator.program} shows that we can use the tester program~$T_0(\Pi)$ as a generator program for~$\Pi$.
With some preprocessing and the differences due to the different semantics implemented in~\EPASP---K15 vs G94---this is the strategy followed by this solver.

\paragraph{Generator with consistency constraints.}
\citet{cafagarosc20a} noted that some stable models of~$T_0(\Pi)$ can never correspond to a worldview of~$\Pi$.
Consider, for instance, the single rule program\newprogramph\label{prg:most.basic.problem}
\begin{align}
    b &\leftarrow \K a
    \tag{\program{\ref{prg:most.basic.problem}}}
\end{align}
and its corresponding~$T_0(\program{\ref{prg:most.basic.problem}})$ program
\begin{align*}
    b &\leftarrow ka
    &\hspace{10pt}
    \{ ka \} &\leftarrow
\end{align*}
which has two stable models~$\emptyset$ and~$\{b, ka\}$.
The second stable model cannot correspond to a worldview~$\wv$ of the original program with~${k(\wv) = \{ka \}}$ because, for every non\nobreakdash-empty set of interpretations~$\wv$ such that~$k(\wv) = \{ka \}$, the empty set is the only stable model of~$\Pi^\wv$.
Hence, to be a worldview~$\wv$ must satisfy~$\wv \not\models \K a$, which is a contradiction with~$k(\wv) = \{ka \}$.
Using this observation, 
\citet{cafagarosc20a} introduced a more refined generator program that reduces the number of candidates by adding constraints that ensure that the stable models of the generator program that contain~$\kk{a}$ must also contain~$a$. 

By~$G_0(\Pi)$ we denote the program obtained from~$T_0(\Pi)$ by adding a constraint of the form
\begin{gather}
    \bot \leftarrow \kk{a} \wedge \neg a
    \label{eq:consistency.constraints}
\end{gather}
for each atom~${\kk{a} \in K(\Pi)}$.
Continuing with our running example, we can see that~$\{b, ka\}$ does not satisfy~\eqref{eq:consistency.constraints} and, thus, it is not a stable model of~$G_0(\program{\ref{prg:most.basic.problem}})$.
%

\begin{proposition}\label{prop:basic.generator.program}
    $G_0(\Pi)$ is a generator program for~$\Pi$.
\end{proposition}

%
%
Proposition~\ref{prop:basic.generator.program} shows that we can reduce the number of candidates by adding consistency constraints of the form of~\eqref{eq:consistency.constraints} to the generator program.
If we consider programs in normal form, Algorithm~\ref{alg:guess-and-check} with generator~$G_0(\Pi)$ and tester~$T_0(\Pi)$ is the core of the epistemic solver~\eclingo.
However, \eclingo\ does not compute the normal form of a program, so its algorithm is slightly more complicated than Algorithm~\ref{alg:guess-and-check}.
Hence, this is a proof \eclingo\ is correct for programs in normal form%
\footnote{\citet{cafagarosc20a} described the algorithm of \eclingo\ but they did not show its correctness.}%
\!\!.

\begin{corollary}
    Algorithm~\ref{alg:guess-and-check} with generator~$G_0(\Pi)$ and tester~$T_0(\Pi)$ correctly computes the worldviews of~$\Pi$.
\end{corollary}

\section{Generator with Epistemic Propagation}

In a generate-and-test approach, the tester needs to check every stable model of the guess program.
Each of these checks requires a linear amount of calls to an answer set solver to compute the corresponding cautions consequences.
Each of these calls is $\Sigma^P_2$\nobreakdash-complete.
%
Hence, reducing the number of tester checks is crucial to achieve a fast solver.
In the previous section, we saw how we can reduce the number of candidates by adding consistency constraints to the generator program.
In this section, we introduce a new generator program that can exponentially reduce the number of candidates by propagating the consequences of epistemic literals in the generator program.

\begin{example}\label{ex:propagation}
Let us consider the following program:\newprogramph\label{prg:propagation}
\begin{gather}
    \begin{aligned}
    a_i &\leftarrow \neg \K \nott{a_i} &\text{for } 0 \leq i \leq n
    \\
    \nott{a_i} &\leftarrow \neg a_i &\text{for } 0 \leq i \leq n
    \\
    g &\leftarrow a_i &\text{for } 0 \leq i \leq n
    \\
    \bot &\leftarrow \K g
\end{aligned}\tag{$\program{\ref{prg:propagation}}^n$}
\end{gather}
The two rules first are the result of normalizing rule ${a_i \leftarrow \neg \K \neg a_i}$, usually used as a kind of subjective choice that generates worldviews~$[\emptyset]$ and~$[\{a_i\}]$.
This program has a unique worldview~$[\emptyset]$.
Let us now consider its guess program~$G_0(\program{\ref{prg:propagation}}^n)$:
\begin{align*}
    a_i &\leftarrow \neg \kknott{a_i} &\text{for } 0 \leq i \leq n
    \\
    \nott{a_i} &\leftarrow \neg a_i &\text{for } 0 \leq i \leq n
    \\
    g &\leftarrow a_i &\text{for } 0 \leq i \leq n
    \\
    \bot &\leftarrow \kk{g}
    \\
    \{ \kknott{a_i} \} &\leftarrow  &\text{for } 0 \leq i \leq n
    \\
    \{ \kk{g} \} &\leftarrow 
    \\
    \bot &\leftarrow \kknott{a_i} \wedge \neg \nott{a_i} &\text{for } 0 \leq i \leq n
    \\
    \bot &\leftarrow \kk{g} \wedge \neg g
\end{align*}
This candidate generator program has $2^{n}$ stable models of the form
$$M \cup \{ \nott{a_i} \mid a_i \notin M \} \cup \{ \kknott{a_i} \mid a_i \notin M \} \cup \{g \mid M \neq \emptyset \}$$
with~$M \subseteq \{a_1, \dotsc, a_n\}$.
Only the stable model corresponding to~$M = \emptyset$ leads to a worldview of the original program, while the other~$2^{n}-1$ stable models do not correspond to any worldview.
\end{example}

We can achieve an exponential speed\nobreakdash-up if we can identify the candidates that do not pass the tester check within the generator program.
We can observe that, in every worldview~$\wv$ of~$\program{\ref{prg:propagation}}^n$, if~$\wv \models \neg\K \nott{a_i}$ for any~$1 \leq i \leq n$, then every~$I \in \wv$ must satisfy~$I\models a_i$ because of the first rule.
Hence, every~$I \in \wv$ must satisfy~$I \models g$ because of the third rule, and we obtain~$\wv \models \K g$.
This contradicts the constraint~$\bot \leftarrow \K g$ and, thus, no worldview can satisfy~$\neg\K \nott{a_i}$.
We can implement this reasoning within the generator program by adding rules:
\begin{align*}
    \kp{a_i} &\leftarrow \neg \kknott{a_i} &\text{for } 1 \leq i \leq n
    \\
    \kp{g} &\leftarrow \kp{a_i} &\text{for } 1 \leq i \leq n
    \\
    \bot &\leftarrow \kp{g} \wedge \neg \kk{g}
\end{align*}
The new fresh atoms~$\kp{a_i}$ and~$\kp{g}$ are respectively used to represent that~$\K a_i$ and~$\K g$ are a consequence propagated from the choices made in the generator program.
The meaning of~$\kp{g}$ is different from the meaning of atom~$\kk{g}$ that represents the choice made in the generator program rather than an inferred consequence.
The unintended candidates are removed by the constraint that ensures that if we conclude~$\kp{g}$, then~$\kk{g}$ must also hold.

Let us now describe the general idea of the propagation of epistemic literals by introducing a new guess program~$G_1(\Pi)$.
For each atom~${a \in \At{\Pi}}$, we introduce new fresh atoms~$\kp{a}$ and~$\kpnott{a}$ that respectively represent that~$\K a$ and~$\K \neg a$ are consequences propagated from the choices made in the generator program.
We also introduce a new fresh atom~$\kpnott{r_j}$ for each rule~$r_j$ of~$\Pi$ that represents that the body of~$r_j$ is not satisfied by any interpretation of the worldview corresponding to the current candidate.
We use the following notation.
If~$L$ is a literal, then~$\overline{L}$ is the complement of~$L$, that is
~$\overline{L}$ is~$\neg L$ if~$L$ is an atom or a subjective atom,
and~$\overline{L}$ is~$L'$ if~$L$ is of the form~$\neg L'$.
Furthermore, if~$L$ is a subjective literal, then~$\kk{L}$ 
is a literal obtained by replacing every occurrence of the form~$\K a$ by~$\kk{a}$.

By~$\kp{(\Pi)}$, we denote the program containing rule
\begin{gather}
    \begin{aligned}
        \kp{a_1} &\span\span\leftarrow 
        \kp{a_2},\dotsc,\kp{a_k},
        \kpnott{a_{k+1}},\dotsc,\kpnott{a_m},
        \\
        &\ \ \phantom{\leftarrow}\kk{L_{m+1}},\dotsc,\kk{L_n}
    \end{aligned}
    \label{eq:propagation.rule.regular}
\end{gather}
for every rule~$r_i \in \Pi$ of the form
\begin{align}
    a_1 &\leftarrow 
    a_2,\!\dotsc\!,a_k,
    \neg a_{k+1},\!\dotsc\!,\neg a_m,
    {L_{m+1}},\!\dotsc\!,{L_n}
    \label{eq:propagation.rule.regular}
\end{align}
with each~$L_j$ being a subjective literal;
plus rules of the form
\begin{align*}
    \kpnott{r_i} &\leftarrow \kpnott{a_j} &&\qquad\text{for } 1 \leq j \leq k
    \\
    \kpnott{r_i} &\leftarrow \kp{a_j} &&\qquad\text{for } l+1 \leq j \leq m
    \\
    \kpnott{r_i} &\leftarrow \overline{\kk{L_j}} &&\qquad\text{for } m+1 \leq j \leq n
\end{align*}
for every rule~$r_i \in \Pi$ of the form
\begin{align}
    H &\leftarrow 
    a_2,\!\dotsc\!,a_k,
    \neg a_{k+1},\!\dotsc\!,\neg a_m,
    {L_{m+1}},\!\dotsc\!,{L_n}
    \label{eq:propagation.rule.regular2}
\end{align}
with~$H$ of the form~$a_1 \vee \dotsc \vee a_l$
and each~$L_j$ a subjective literal;
plus a rule of the form
\begin{align}
    \kpnott{a} &\leftarrow \kpnott{r_1},\dotsc,\kpnott{r_k}
    \label{eq:propagation.rule.completion}
\end{align}
for each atom~$a \in \At{\Pi}$ and where~$r_1,\dotsc,r_k$ are the rules of~$\Pi$ that have~$a$ in the head.
By~$G_1(\Pi)$ we denote the program obtained from~$G_0(\Pi)$ by adding the rules of~$\kp{(\Pi)}$ plus a \emph{consistency constraint} of the form
\begin{align}
    \bot &\leftarrow \kp{a} \wedge \neg \kk{a}
    \label{eq:consistency.constraints.propagation.p}
\end{align}
for each atom~$\kk{a} \in K(\Pi)$.

\begin{lemma}\label{lem:propagation.generator.program}
    Let~$\wv$ be a non\nobreakdash-empty set of interpretations, and~$M$ be an interpretation such that~$k(\wv) = k(M)$.
    Then,
    \begin{itemize}
        \item If~$M$ is a stable model of~$G_0(\Pi)$, then there is a unique stale model~$M'$ of~$G_0(\Pi) \cup \kp{(\Pi)}$ of the form
        $M' = M \cup M^p \cup M^n$ with
        \begin{align}
            M^p &\subseteq \{ \kp{a} \hspace{5pt} \mid \SM{\Pi^\wv} \models \K a \}
            \label{eq:1:prop:propagation.generator.program}
            \\
            M^n &\subseteq \{ \kpnott{a} \mid \SM{\Pi^\wv} \models \K \neg a \}
            \label{eq:2:prop:propagation.generator.program}
        \end{align}
        \item If $M'$ is a stable model of program~${G_0(\Pi) \cup \kp{(\Pi)}}$, then ${M' \cap \At{G_0(\Pi)}}$ is a stable model of~$G_0(\Pi)$.
    \end{itemize}
\end{lemma}

\vspace*{2pt}

\begin{mtheorem}\label{thm:propator.correctness}
    $G_1(\Pi)$ is a generator program for~$\Pi$ and, thus, Algorithm~\ref{alg:guess-and-check} with generator~$G_1(\Pi)$ and tester~$T_0(\Pi)$ correctly computes the worldviews of~$\Pi$.
\end{mtheorem}

\begin{proof}
\emph{Condition 1 for generator program.}
Pick a worldview~$\wv$ of~$\Pi$.
By Proposition~\ref{prop:basic.generator.program},  there is a unique stable model~$M$ of~$G_0(\Pi)$ satisfying~${k(\wv) = k(M)}$ and~$\restr{M}{\At{\Pi}} \in \wv$.
By Lemma~\ref{lem:propagation.generator.program}, there is a unique stable model~$M'$ of~$G_0(\Pi) \cup \kp{(\Pi)}$ of the form~$M' = M \cup M^p \cup M^n$ with~$M^p$ and~$M^n$ satisfying~\eqref{eq:1:prop:propagation.generator.program} and~\eqref{eq:2:prop:propagation.generator.program}.
Furthermore, ${\restr{M'}{\At{\Pi}} = \restr{M}{\At{\Pi}} \in \wv}$ and~${\kk(M) = \kk(M')}$.
It only remains to be shown that~$M'$ is a stable model of~$G_1(\Pi)$, that is, that it satisfies the consistency constraints of the form of~\eqref{eq:consistency.constraints.propagation.p}.
Pick~$\kp{a} \in M^p$.
Then, $\SM{\Pi^\wv} \models \K a$ and, thus, $\wv \models \K a$ and~$\kk{a} \in M$ and, thus, the constraints is satisfied.
%
\\[3pt]
\emph{Condition 2 for generator program.}
Pick now a stable model of~$G_1(\Pi)$ and let~$\wv$ be any set of interpretations such that~$k(\wv) = k(M)$.
By Lemma~\ref{lem:propagation.generator.program}, $M \cap \At{G_0(\Pi)}$ is a stable model of~$G_0(\Pi)$ and, since~$G_0(\Pi)$ is a generator program for~$\Pi$, it follows that~$\restr{M}{\At{\Pi}} \in \SM{\Pi^\wv}$.
\\[3pt]
Therefore~$G_1(\Pi)$ is a generator program and, by Theorem~\ref{thm:guess-and-check}, Algorithm~\ref{alg:guess-and-check} with generator~$G_1(\Pi)$ and tester~$T_0(\Pi)$ correctly computes the worldviews of~$\Pi$.
\end{proof}

This Theorem shows that we can use~$G_1(\Pi)$ as a generator program for~$\Pi$ instead of~$G_0(\Pi)$.
Let us show now that using~$G_1(\Pi)$ instead of~$G_0(\Pi)$ is actually beneficial when using Algorithm~\ref{alg:guess-and-check}.


\begin{proposition}\label{prop:propagation.generator.program}
    The following properties hold:
    \begin{enumerate}
        \item $|\SM{G_1(\Pi)}| \leq |\SM{G_0(\Pi)}|$ for every program~$\Pi$,
        \item for every program~$\Pi$ we can compute a stable model of $G_1(\Pi)$ from a stable model of~$G_0(\Pi)$ in linear time, and
        \item there is a family of programs~$\Pi_1, \Pi_2, \dotsc$ such that $2^{i}|\SM{G_1(\Pi_i)}| \leq |\SM{G_0(\Pi_i)}|$ and~${\mathcal{O}(s(\Pi_i)) = i}$.
    \end{enumerate}
    where~$s(\Pi)$ denotes the size of program~$\Pi$.
\end{proposition}

The first two properties show that using~$G_1(\Pi)$ instead of~$G_0(\Pi)$ only incurs a linear overhead in computing the stable models of the generator program.
In particular, the first property ensures that the loop between lines~\ref{alg.generator.loop.ini} and~\ref{alg.generator.loop.end} of Algorithm~\ref{alg:guess-and-check} is never required to do more iterations when using~$G_1(\Pi)$ instead of~$G_0(\Pi)$.
The third property shows that using~$G_1(\Pi)$ can reduce the number of candidates that need to be checked exponentially.
As an example of such a family of programs, consider programs~$(\Pi_2^n)_{n \geq 1}$ of Example~\ref{ex:propagation}.

Lemma~\ref{lem:propagation.generator.program} also implies an interesting property that allows us to skip some of the tester checks.
By~\eqref{eq:1:prop:propagation.generator.program} and~\eqref{eq:2:prop:propagation.generator.program}, it follows that any stable model~$M$ of~$G_1(\Pi)$ that satisfies the following two conditions corresponds to a worldview of~$\Pi$ and, thus, the tester check can be skipped:
\begin{enumerate}
    \item every~$\kk{a} \in M$ satisfies~$\kp{a} \in M$, and
    \item every~$\kk{a} \notin M$ satisfies~$\kpnott{a} \in M$.
\end{enumerate}
Hence, when using~$G_1(\Pi)$ as a generator program, function~\texttt{Test}$(T(\Pi),\, M)$ first checks these two conditions, and only if any of them fails it checks condition~\eqref{eq:1:thm:guess-and-check}.
Checking these two conditions is a linear time operation while checking condition~\eqref{eq:1:thm:guess-and-check} requires a linear amount of calls to a $\Sigma^P_2$\nobreakdash-complete oracle.

\section{Implementation and Experimental Evaluation}

\paragraph*{Implementation.}
We implement a new solver for epistemic logic programs based on Algorithm~\ref{alg:guess-and-check}.
Our implementation is built on top of version~5.7 of the ASP solver~\texttt{clingo}~\cite{gekakasc17a} using Python and ASP.
We extend the language of \texttt{clingo} by allowing literals of the form $\texttt{\&k\{} L \texttt{\}}$ in the body of rules, with $L$ being a literal of the forms~$A$, $\texttt{not\,}A$, or $\texttt{not\,not\,} A$ for some atom~$A$ in the usual syntax of~\texttt{clingo}.
%
%
An interesting side effect of using \texttt{clingo} in this way is that our solver accepts most of the features of \texttt{clingo} such as aggregates, choice rules, intervals, pools, etc.
Another important aspect of our implementation is that it easily allows us to use different generator and tester programs.
To achieve this, we rely on \emph{metaprogramming}, where the ground program is reified as a set of facts and we can use an ASP program to produce the generator and tester programs~\cite{karoscwa21a}.
As a result, we can change the generator and tester programs by only changing the ASP metaprogram.
This is released as a new version of~\texttt{eclingo}~(\url{https://github.com/potassco/eclingo}).

\paragraph*{Experimental Evaluation.}

For the experimental evaluation, we use the well-established benchmark suite by~\citet{solekale17a}.
It consists of three problems: 
the \emph{Eligibility} problem that represents reasoning in disjunctive databases~\cite{gelfond91a},
and the \emph{Yale Shooting} and \emph{Bomb in the Toilet} problems that are instances of conformant planning.
The original suite consists of 58 instances (25 of the Eligibility problem, 7 of the Yale problem and 26 of the Bomb problem).
We expand this suite by adding 268 new instances for a total of 326 instances (145 Eligibility problems, 7 Yale problems and 174 for Bomb problems).
The new instances 
are automatically generated by increasing the number of students in the Eligibility problem and the number of packages and toilets in the Bomb problem. 
We set a timeout of 600s and ran our experiments on a machine powered by an Intel Core Processor (Broadwell) with 12 CPUs running at 2095.078 MHz and 100GB of RAM.
The OS is Red Hat Enterprise Linux Server~7.9.
The code to repeat the benchmarks can be found at~\url{https://github.com/krr-uno/eclingo-benchmark}.

\begin{figure}[t]%
    \centering
\begin{table}[H]
    \small\centering
    \scalebox{1}{
    \begin{tabular}{|c|c|c|c|c|}
    \hline
    & \multicolumn{2}{c|}{$G_1$} & \multicolumn{2}{c|}{$G_0$}
    \\
    \hline
    Benchmark&\texttt{\#solved} & \texttt{time} & \texttt{\#solved} & \texttt{time}\\
    \hline
    Eligibility & 145 & 4.03 & 145 & 5.38
    \\
    Yale & 7 & 0.009 & 7 & 0.112
    \\
    Bomb & 131 & 245.66 & 13 & 558.40
    \\
    \hline
    All & 283 & 132.912 & 166 & 300.435
    \\
    \hline
    \end{tabular}%
    } 
    \caption{Instances for each benchmark solved by the different versions of the generator program.}
    \label{table1}
\end{table}%
\vspace{-10pt}%
\end{figure}
\begin{figure}[t]
    \centering
    \resizebox{0.48\textwidth}{!}{%
    \pgfplotsset{compat = newest}

\begin{tikzpicture}
    \begin{axis}[
        width=0.5\textwidth,
        ylabel={Time (Seconds)},
        xlabel={Nº of Instances solved},
        ymin=0, ymax=603,
        xmin=0, xmax=353,
        ytick={0, 60, 120, 180, 240, 300, 360, 420, 480, 540, 600},
        xtick={50, 100, 150, 200, 250, 300, 350},
        legend style={
                    anchor=north,
                    legend columns=1,
                    nodes={scale=1, transform shape},
                    /tikz/every even column/.append style={column sep=0.5cm}
                },
        legend pos=north west,
        ymajorgrids=true,
        grid style=dashed,
    ]
    
            \addplot[
            color=blue,
            mark=*,
            ]
            coordinates {
            (1,0.0)(2,0.0)(3,0.0)(4,0.0)(5,0.0)(6,0.0)(7,0.0)(8,0.0)(9,0.0)(10,0.001)(11,0.001)(12,0.001)(13,0.001)(14,0.001)(15,0.001)(16,0.002)(17,0.002)(18,0.002)(19,0.002)(20,0.002)(21,0.002)(22,0.003)(23,0.003)(24,0.003)(25,0.003)(26,0.0035)(27,0.004)(28,0.004)(29,0.004)(30,0.004)(31,0.0045000000000000005)(32,0.0055)(33,0.0055)(34,0.007)(35,0.0075)(36,0.0075)(37,0.0075)(38,0.008)(39,0.008)(40,0.0085)(41,0.009)(42,0.009)(43,0.0095)(44,0.0095)(45,0.01)(46,0.01)(47,0.011)(48,0.012)(49,0.013)(50,0.014)(51,0.014)(52,0.014499999999999999)(53,0.015)(54,0.016)(55,0.016)(56,0.0165)(57,0.017)(58,0.017)(59,0.017)(60,0.018)(61,0.0185)(62,0.019)(63,0.0195)(64,0.02)(65,0.02)(66,0.022)(67,0.022)(68,0.0225)(69,0.0245)(70,0.025)(71,0.025500000000000002)(72,0.026)(73,0.0265)(74,0.027)(75,0.0275)(76,0.0285)(77,0.029)(78,0.029)(79,0.0295)(80,0.0305)(81,0.035500000000000004)(82,0.0365)(83,0.0405)(84,0.041)(85,0.041999999999999996)(86,0.041999999999999996)(87,0.0435)(88,0.0435)(89,0.045)(90,0.0455)(91,0.046)(92,0.0495)(93,0.0495)(94,0.051500000000000004)(95,0.0545)(96,0.057)(97,0.058499999999999996)(98,0.073)(99,0.0745)(100,0.081)(101,0.0825)(102,0.089)(103,0.0925)(104,0.0965)(105,0.098)(106,0.1005)(107,0.1015)(108,0.1035)(109,0.1095)(110,0.119)(111,0.1395)(112,0.14150000000000001)(113,0.143)(114,0.14400000000000002)(115,0.14550000000000002)(116,0.15)(117,0.162)(118,0.16599999999999998)(119,0.166)(120,0.1735)(121,0.183)(122,0.186)(123,0.187)(124,0.1995)(125,0.209)(126,0.2285)(127,0.23299999999999998)(128,0.2425)(129,0.2525)(130,0.2735)(131,0.3035)(132,0.31)(133,0.325)(134,0.332)(135,0.33399999999999996)(136,0.335)(137,0.3865)(138,0.4165)(139,0.443)(140,0.473)(141,0.511)(142,0.5615)(143,0.576)(144,0.5905)(145,0.699)(146,0.6995)(147,0.709)(148,0.713)(149,0.767)(150,0.7705)(151,0.7945)(152,0.873)(153,0.9430000000000001)(154,1.093)(155,1.209)(156,1.476)(157,1.492)(158,1.6244999999999998)(159,1.669)(160,1.9645)(161,2.074)(162,2.433)(163,2.656)(164,2.7105)(165,2.7554999999999996)(166,2.791)(167,3.2279999999999998)(168,3.426)(169,3.4875)(170,3.7359999999999998)(171,3.945)(172,5.074)(173,5.6205)(174,5.6835)(175,6.1739999999999995)(176,6.858499999999999)(177,7.101)(178,7.1845)(179,7.401)(180,7.484999999999999)(181,7.7295)(182,8.6785)(183,9.5195)(184,9.6065)(185,9.725)(186,9.9785)(187,10.198)(188,11.009)(189,12.105)(190,12.9335)(191,12.9895)(192,13.048)(193,13.741)(194,14.071)(195,14.6875)(196,15.607)(197,15.6495)(198,16.7285)(199,17.546)(200,19.2515)(201,19.4405)(202,20.0225)(203,20.1755)(204,21.346)(205,21.8425)(206,22.506999999999998)(207,22.9685)(208,23.0245)(209,25.990000000000002)(210,27.954500000000003)(211,28.6575)(212,28.861)(213,30.239)(214,30.689999999999998)(215,31.698)(216,33.5935)(217,33.977000000000004)(218,35.495000000000005)(219,36.542500000000004)(220,39.2845)(221,42.457)(222,43.36)(223,43.462500000000006)(224,43.7745)(225,44.855000000000004)(226,45.0345)(227,45.442)(228,45.702)(229,45.7215)(230,48.067)(231,48.7745)(232,48.899)(233,60.973)(234,62.918000000000006)(235,64.174)(236,64.305)(237,64.322)(238,65.587)(239,71.744)(240,76.832)(241,77.42349999999999)(242,78.3575)(243,81.8925)(244,83.0865)(245,85.428)(246,85.60300000000001)(247,89.4085)(248,93.1285)(249,94.14150000000001)(250,94.1545)(251,95.6335)(252,97.836)(253,101.822)(254,105.429)(255,111.1215)(256,114.4745)(257,130.3945)(258,130.39499999999998)(259,132.74099999999999)(260,133.192)(261,141.6215)(262,143.67899999999997)(263,149.081)(264,160.928)(265,167.4155)(266,172.6955)(267,173.2645)(268,178.19)(269,180.7395)(270,205.6015)(271,231.82299999999998)(272,302.8165)(273,364.9225)(274,453.5535)(275,534.2245)(276,600.0)(277,600.0)(278,600.0)(279,600.0)(280,600.0)(281,600.0)(282,600.0)(283,600.0)(284,600.0)(285,600.0)(286,600.0)(287,600.0)(288,600.0)(289,600.0)(290,600.0)(291,600.0)(292,600.0)(293,600.0)(294,600.0)(295,600.0)(296,600.0)(297,600.0)(298,600.0)(299,600.0)(300,600.0)(301,600.0)(302,600.0)(303,600.0)(304,600.0)(305,600.0)(306,600.0)(307,600.0)(308,600.0)(309,600.0)(310,600.0)(311,600.0)(312,600.0)(313,600.0)(314,600.0)(315,600.0)(316,600.0)(317,600.0)(318,600.0)(319,600.0)(320,600.0)(321,600.0)(322,600.0)(323,600.0)(324,600.0)(325,600.0)(326,600.0)(327,600.0)(328,600.0)(329,625.809)(330,628.068)(331,793.9835)(332,829.128)(333,1068.6154999999999)
            };
            \addlegendentry{ $G_1$ }
            
            \addplot[
            color=red,
            mark=x,
            ]
            coordinates {
            (1,0.0)(2,0.0)(3,0.0)(4,0.0)(5,0.0)(6,0.0)(7,0.0)(8,0.0)(9,0.0)(10,0.0)(11,0.0)(12,0.001)(13,0.001)(14,0.001)(15,0.001)(16,0.001)(17,0.001)(18,0.001)(19,0.001)(20,0.001)(21,0.001)(22,0.001)(23,0.001)(24,0.001)(25,0.001)(26,0.001)(27,0.001)(28,0.001)(29,0.001)(30,0.0015)(31,0.0015)(32,0.002)(33,0.002)(34,0.002)(35,0.002)(36,0.002)(37,0.002)(38,0.002)(39,0.002)(40,0.003)(41,0.004)(42,0.004)(43,0.004)(44,0.005)(45,0.005)(46,0.0055)(47,0.007)(48,0.007)(49,0.007)(50,0.0075)(51,0.008)(52,0.008)(53,0.008)(54,0.008)(55,0.009)(56,0.010499999999999999)(57,0.011)(58,0.013)(59,0.013)(60,0.014)(61,0.014)(62,0.017)(63,0.018)(64,0.018000000000000002)(65,0.02)(66,0.02)(67,0.022)(68,0.024)(69,0.027)(70,0.033)(71,0.036)(72,0.0375)(73,0.039)(74,0.0405)(75,0.044)(76,0.0465)(77,0.053)(78,0.053500000000000006)(79,0.064)(80,0.0645)(81,0.0715)(82,0.082)(83,0.08499999999999999)(84,0.087)(85,0.088)(86,0.0885)(87,0.092)(88,0.094)(89,0.095)(90,0.098)(91,0.1005)(92,0.1115)(93,0.127)(94,0.128)(95,0.129)(96,0.135)(97,0.1495)(98,0.1525)(99,0.157)(100,0.165)(101,0.166)(102,0.178)(103,0.181)(104,0.191)(105,0.209)(106,0.2185)(107,0.2215)(108,0.255)(109,0.262)(110,0.279)(111,0.2795)(112,0.295)(113,0.2985)(114,0.3095)(115,0.37)(116,0.40049999999999997)(117,0.4435)(118,0.5335000000000001)(119,0.538)(120,0.642)(121,0.647)(122,0.6685000000000001)(123,0.6795)(124,0.715)(125,0.887)(126,0.967)(127,1.021)(128,1.0495)(129,1.097)(130,1.202)(131,1.3615)(132,1.385)(133,1.5845)(134,1.956)(135,2.324)(136,2.572)(137,2.5865)(138,2.6615)(139,3.304)(140,3.8135)(141,4.407)(142,4.884)(143,4.9185)(144,5.8315)(145,5.94)(146,6.391)(147,6.656000000000001)(148,7.235)(149,8.084499999999998)(150,9.374)(151,9.563500000000001)(152,9.604)(153,9.994)(154,11.94)(155,12.325)(156,12.574)(157,13.189499999999999)(158,19.247500000000002)(159,19.299)(160,21.976)(161,22.747)(162,22.911)(163,25.759)(164,27.314)(165,32.3605)(166,37.239)(167,37.460499999999996)(168,43.175)(169,51.040499999999994)(170,60.0795)(171,182.252)(172,548.401)(173,600.0)(174,600.0)(175,600.0)(176,600.0)(177,600.0)(178,600.0)(179,600.0)(180,600.0)(181,600.0)(182,600.0)(183,600.0)(184,600.0)(185,600.0)(186,600.0)(187,600.0)(188,600.0)(189,600.0)(190,600.0)(191,600.0)(192,600.0)(193,600.0)(194,600.0)(195,600.0)(196,600.0)(197,600.0)(198,600.0)(199,600.0)(200,600.0)(201,600.0)(202,600.0)(203,600.0)(204,600.0)(205,600.0)(206,600.0)(207,600.0)(208,600.0)(209,600.0)(210,600.0)(211,600.0)(212,600.0)(213,600.0)(214,600.0)(215,600.0)(216,600.0)(217,600.0)(218,600.0)(219,600.0)(220,600.0)(221,600.0)(222,600.0)(223,600.0)(224,600.0)(225,600.0)(226,600.0)(227,600.0)(228,600.0)(229,600.0)(230,600.0)(231,600.0)(232,600.0)(233,600.0)(234,600.0)(235,600.0)(236,600.0)(237,600.0)(238,600.0)(239,600.0)(240,600.0)(241,600.0)(242,600.0)(243,600.0)(244,600.0)(245,600.0)(246,600.0)(247,600.0)(248,600.0)(249,600.0)(250,600.0)(251,600.0)(252,600.0)(253,600.0)(254,600.0)(255,600.0)(256,600.0)(257,600.0)(258,600.0)(259,600.0)(260,600.0)(261,600.0)(262,600.0)(263,600.0)(264,600.0)(265,600.0)(266,600.0)(267,600.0)(268,600.0)(269,600.0)(270,600.0)(271,600.0)(272,600.0)(273,600.0)(274,600.0)(275,600.0)(276,600.0)(277,600.0)(278,600.0)(279,600.0)(280,600.0)(281,600.0)(282,600.0)(283,600.0)(284,600.0)(285,600.0)(286,600.0)(287,600.0)(288,600.0)(289,600.0)(290,600.0)(291,600.0)(292,600.0)(293,600.0)(294,600.0)(295,600.0)(296,600.0)(297,600.0)(298,600.0)(299,600.0)(300,600.0)(301,600.0)(302,600.0)(303,600.0)(304,600.0)(305,600.0)(306,600.0)(307,600.0)(308,600.0)(309,600.0)(310,600.0)(311,600.0)(312,600.0)(313,600.0)(314,600.0)(315,600.0)(316,600.0)(317,600.0)(318,600.0)(319,600.0)(320,600.0)(321,600.0)(322,600.0)(323,600.0)(324,600.0)(325,600.0)(326,600.0)(327,600.0)(328,600.0)(329,600.0)(330,600.0)(331,600.0)(332,600.0)(333,600.0)
            };
            \addlegendentry{ $G_0$ }
            
    \end{axis}
\end{tikzpicture}
    
    } 
    \caption{Comparison of our implementation of Algorithm~\ref{alg:guess-and-check} with two different versions of the generator program.}
    \label{fig2} 
    \vspace{-10pt}%
\end{figure}
\begin{table*}[t]
    \small\centering
    \scalebox{0.72}{
    \begin{tabular}{|c|c|c|c|c|c|c|c|c|c|c|c|c||c|c|}
    \hline
    & \multicolumn{2}{c|}{$G_1$} 
    & \multicolumn{2}{c|}{$G_0$} 
    & \multicolumn{2}{c|}{\texttt{eclingo}} 
    & \multicolumn{2}{c|}{\texttt{EP-ASP}} 
    & \multicolumn{2}{c|}{\texttt{selp}} 
    & \multicolumn{2}{c||}{\texttt{elp2qasp}} 
    & \multicolumn{2}{c|}{$\texttt{EP-ASP}^{\mathit{se}}$}
    \\
    \hline
    Benchmark 
    & \texttt{\#solved} 
    & \texttt{time} 
    & \texttt{\#solved} 
    & \texttt{time} 
    & \texttt{\#solved} 
    & \texttt{time} 
    & \texttt{\#solved} 
    & \texttt{time} 
    & \texttt{\#solved} 
    & \texttt{time} 
    & \texttt{\#solved} 
    & \texttt{time} 
    & \texttt{\#solved} 
    & \texttt{time}\\
    \hline
    Eligibility 
    & 145 
    & 17.982 
    & 145 
    & 18.621 
    & 124 
    & 134.244  
    & 28 
    & 523.102 
    & 60 
    & 380.366 
    & 6 
    & 577.303 
    & - 
    & -
    \\
    Yale 
    & 7 
    & 0.377 
    & 7 
    & 0.469 
    & 7 
    & 0.157 
    & 7 
    & 30.377 
    & 5 
    & 223.296 
    & 7 
    & 1.926 
    & 7 
    & 0.109
    \\
    Bomb 
    & 131 
    & 246.039 
    & 15 
    & 556.396 
    & 17 
    & 561.52 
    & 45 
    & 471.759
    & 8 
    & 572.809  
    & 8 
    & 572.663 
    & 86 
    & 327.900
    \\
    \hline
    All 
    & 283 
    & 139.328 
    & 166 
    & 306.350 
    & 148 
    & 359.42 
    & 80 
    & 485.118 
    & 73 
    & 479.708 
    & 21 
    & 562.472 
    & 121 
    & 407.684
    \\
    \hline
    \end{tabular}%
    } 
    \caption{Instances for each benchmark solved by the different epistemic solvers}
    \label{table2}
    \vspace{-10pt}%
\end{table*}%

\paragraph{Results.}
We first compare our implementation of Algorithm~\ref{alg:guess-and-check} with two different versions of the generator program, denoted~$G_0$ and~$G_1$.
Table~\ref{table1} shows the number of instances solved within the timeout and the average solving time by each version of the generator program.
On computing average times, instances that did not solve within the timeout were assigned a time of 600s.
Figure~\ref{fig2} reports the cactus plot for these two versions of our solver. 
Using~$G_1$ solves more instances than using~$G_0$ with any timeout greater than 0.546 seconds.
There are 49 easy instances solved under that timeout where using~$G_0$ outperforms~$G_1$.

\begin{figure}[t]
    \small\centering
    \resizebox{0.48\textwidth}{!}{%
\begin{tikzpicture}
\begin{axis}[
    width=0.5\textwidth,
    ylabel={Time(Seconds)},
    xlabel={Nº of Instances solved},
    ymin=0, ymax=603,
    xmin=0, xmax=350,
    ytick={0, 60, 120, 180, 240, 300, 360, 420, 480, 540, 600},
    xtick={50, 100, 150, 200, 250, 300, 350},
    legend style={
                at={(0.865,0.165)},
                anchor=center,
                legend columns=1,
                nodes={scale=0.6, transform shape},
                /tikz/every even column/.append style={column sep=0.5cm}
            },
    ymajorgrids=true,
    grid style=dashed,
]
\input{solver_comp_g1.tex}
\input{solver_comp_eclingo.tex}
\input{solver_comp_epasp.tex}      
\input{solver_comp_epasp_noplanning}
\input{solver_comp_selp.tex}
\input{solver_comp_elp2qasp.tex}
\end{axis}
\end{tikzpicture}
    } 
    \caption{Comparison of all Epistemic Solvers Eligible, Yale and Bomb Problems.}
    \label{fig1} 
    \vspace{-18pt}%
\end{figure}
We also compare our implementation with the state-of-the-art epistemic solvers~\texttt{EP-ASP}~\cite{solekale17a},
the previous version of~\texttt{eclingo}%
\footnote{\url{https://github.com/potassco/eclingo/releases/tag/v0.2.1}}%
~\cite{cafagarosc20a},
\texttt{selp}%
\footnote{\url{https://dbai.tuwien.ac.at/proj/selp/}}
~\cite{bimowo20a}, 
and~\texttt{elp2qasp}%
\footnote{\texttt{elp2qasp} was shared with us by its authors.}
~\cite{fabmor23a}.
For \texttt{EP-ASP}, we use the version%
\footnote{\url{https://github.com/mmorak/EP-ASP}}
created by~\citet{fabmor23a} to run with \texttt{clingo}~5.4.
Table~\ref{table2} shows the number of instances solved within the timeout and the average expended time by each solver.
In this case, expended time comprises both grounding and solving because we do not have access to the solving time of all solvers.
Our implementation of Algorithm~\ref{alg:guess-and-check} with the generator program~$G_1$ solves the most instances, dominating the other solvers across all benchmarks.
It solves 91\% more instances than \texttt{eclingo} and more than double than~\texttt{EP-ASP}.
When looking at average times, our solver achieves a speed\nobreakdash-up of \speedup{3.3x} and \speedup{3.8x} when compared with the previous version of~\texttt{eclingo} and~\texttt{EP-ASP}, respectively%
\footnote{The speed\nobreakdash-up is computed by dividing the total time of the solver divided by the total time of our solver.
Only instances solved by at least one solver are considered.
}.
When looking at individual benchmarks, we can see that \texttt{eclingo} solves more instances than~\texttt{EP-ASP} mostly thanks to the Eligibility problem, where it solves more than 4 times the number of instances.
When looking at the Bomb problem things reverse and~\texttt{EP-ASP} solves \speedup{2.6x} more instances than~\texttt{eclingo}.
Our new solver successfully solves all 145 instances of the Eligibility problem and almost \speedup{7.7x} more instances than~\texttt{EP-ASP} in the Bomb problem.
When looking at average times, our solver is \speedup{7.5} times faster than \texttt{eclingo} in the Eligibility problem and \speedup{2.4x} faster than~\texttt{EP-ASP} in the Bomb problem.
The Yale Problem seems quite easy for state-of-the-art solvers, with all but \texttt{selp} solving all~7 instances, and both versions of~\texttt{eclingo} solving every instance in less than a second.
It is also worth mentioning that among state\nobreakdash-of\nobreakdash-the\nobreakdash-art solvers, generate\nobreakdash-and\nobreakdash-test-based solvers outperform translational\nobreakdash-based solvers, with the three top solvers being generate\nobreakdash-and\nobreakdash-test-based and the two bottom ones being translational-based.
Finally, \texttt{EP-ASP} has a conformant planning mode that uses domain-specific heuristics that gives it a significant advantage on the Yale and Bomb problems.
This is reported in Table~\ref{table2} and Figure~\ref{fig1} as~$\texttt{EP-ASP}^{\mathit{se}}$.
Even with this advantage, our solver is \speedup{1.5x} faster in the Bomb problem and it solves \speedup{50\%} more of its instances.

\section{Conclusions}\label{sec:conclusions}

We present a general framework to study generate-and-test-based solvers for epistemic logic programs (ELPs).
In this framework, we first compute the normal form of the epistemic program.
Then, we check all the stable models of a generator program by computing the cautious consequences of a tester program.
We provided sufficient conditions on the generator and tester programs for the correctness of the solvers built using this algorithm.
We instantiate this framework with generator programs corresponding to the existing generate-and-test-based solvers (\EPASP\ and \eclingo) and introduce and new generator program based on the idea of propagating epistemic consequences.
We formally prove that this new generator program can exponentially reduce the number of candidates while only incurring a linear overhead.
We experimentally evaluate our new solver and show that it outperforms existing solvers by achieving a~\speedup{3.3x} speed-up and solving 91\% more instances on well-known benchmarks.
Even when comparing with~$\EPASP^{\mathit{se}}$ on conformant planning problems our solver is \speedup{1.5x} faster and solves~\speedup{50\%} more instances, despite the latter using domain-specific heuristics and our solver being based only on general-purpose algorithms for ELPs.
For future work, we plan to study the impact of splitting~\cite{cafafa21a} in worldview computation and the use of general-purpose and domain-specific heuristics to further improve the performance of our solver.
For domain-specific heuristics,
we will focus on conformant planning and adapt the heuristics used by~$\EPASP^{\mathit{se}}$ and the one used for classical planning in ASP~\cite{gekaotroscwa13a}.

\section*{Acknowledgements}
This research is partially supported by NSF CAREER award 2338635.
Any opinions, findings, and conclusions or recommendations expressed in this material are those of the authors and do not necessarily reflect the views of the National Science Foundation.
\bibliography{krr,procs}

\begin{thebibliography}{35}
\providecommand{\natexlab}[1]{#1}

\bibitem[{Balduccini, Lierler, and Woltran(2019)}]{lpnmr19}
Balduccini, M.; Lierler, Y.; and Woltran, S., eds. 2019.
\newblock \emph{Proceedings of the Fifteenth International Conference on Logic
  Programming and Nonmonotonic Reasoning (LPNMR'19)}, volume 11481 of
  \emph{Lecture Notes in Artificial Intelligence}. Springer-Verlag.

\bibitem[{Balduccini and Son(2011)}]{mg65}
Balduccini, M.; and Son, T., eds. 2011.
\newblock \emph{Logic Programming, Knowledge Representation, and Nonmonotonic
  Reasoning: Essays Dedicated to {M}ichael {G}elfond on the Occasion of his
  65th Birthday}, volume 6565 of \emph{Lecture Notes in Computer Science}.
  Springer-Verlag.

\bibitem[{Bichler, Morak, and Woltran(2020)}]{bimowo20a}
Bichler, M.; Morak, M.; and Woltran, S. 2020.
\newblock selp: {A} Single-Shot Epistemic Logic Program Solver.
\newblock \emph{Theory and Practice of Logic Programming}, 20(4): 435--455.

\bibitem[{Brewka, Eiter, and Truszczy{\'n}ski(2011)}]{breitr11a}
Brewka, G.; Eiter, T.; and Truszczy{\'n}ski, M. 2011.
\newblock Answer set programming at a glance.
\newblock \emph{Communications of the {ACM}}, 54(12): 92--103.

\bibitem[{Cabalar, Fandinno, and {Fari{\~n}as del
  Cerro}(2019{\natexlab{a}})}]{cafafa19a}
Cabalar, P.; Fandinno, J.; and {Fari{\~n}as del Cerro}, L. 2019{\natexlab{a}}.
\newblock Founded World Views with Autoepistemic Equilibrium Logic.
\newblock In  \cite{lpnmr19}, 134--147.

\bibitem[{Cabalar, Fandinno, and {Fari{\~n}as del
  Cerro}(2019{\natexlab{b}})}]{cafafa19b}
Cabalar, P.; Fandinno, J.; and {Fari{\~n}as del Cerro}, L. 2019{\natexlab{b}}.
\newblock Splitting Epistemic Logic Programs.
\newblock In  \cite{lpnmr19}, 120--133.

\bibitem[{Cabalar, Fandinno, and {Fari{\~n}as del Cerro}(2020)}]{cafafa20a}
Cabalar, P.; Fandinno, J.; and {Fari{\~n}as del Cerro}, L. 2020.
\newblock Autoepistemic Answer Set Programming.
\newblock \emph{Artificial Intelligence}, 289: 103382.

\bibitem[{Cabalar, Fandinno, and {Fari{\~n}as del Cerro}(2021)}]{cafafa21a}
Cabalar, P.; Fandinno, J.; and {Fari{\~n}as del Cerro}, L. 2021.
\newblock Splitting Epistemic Logic Programs.
\newblock \emph{Theory and Practice of Logic Programming}, 21: 296--316.

\bibitem[{Cabalar et~al.(2020)Cabalar, Fandinno, Garea, Romero, and
  Schaub}]{cafagarosc20a}
Cabalar, P.; Fandinno, J.; Garea, J.; Romero, J.; and Schaub, T. 2020.
\newblock eclingo: A solver for Epistemic Logic Programs.
\newblock \emph{Theory and Practice of Logic Programming}, 20(5): 834--847.
\newblock \url{https://github.com/potassco/eclingo}.

\bibitem[{Dantsin et~al.(2001)Dantsin, Eiter, Gottlob, and
  Voronkov}]{daeigovo01a}
Dantsin, E.; Eiter, T.; Gottlob, G.; and Voronkov, A. 2001.
\newblock Complexity and Expressive Power of Logic Programming.
\newblock \emph{{ACM} Computing Surveys}, 33(3): 374--425.

\bibitem[{Faber and Morak(2023)}]{fabmor23a}
Faber, W.; and Morak, M. 2023.
\newblock Evaluating Epistemic Logic Programs via Answer Set Programming with
  Quantifiers.
\newblock In \emph{Proceedings of the Thirty-seventh National Conference on
  Artificial Intelligence (AAAI'23)}, 6322--6329. {AAAI} Press.

\bibitem[{Faber, Morak, and Chrpa(2021)}]{famoch21a}
Faber, W.; Morak, M.; and Chrpa, L. 2021.
\newblock Determining Action Reversibility in {STRIPS} Using Answer Set and
  Epistemic Logic Programming.
\newblock \emph{Theory and Practice of Logic Programming}, 21(5): 646--662.

\bibitem[{Faber and Woltran(2011)}]{fabwol11a}
Faber, W.; and Woltran, S. 2011.
\newblock Manifold Answer-Set Programs and Their Applications.
\newblock In  \cite{mg65}, 44--63.

\bibitem[{Fandinno, Faber, and Gelfond(2022)}]{fafage22a}
Fandinno, J.; Faber, W.; and Gelfond, M. 2022.
\newblock Thirty years of Epistemic Specifications.
\newblock \emph{Theory and Practice of Logic Programming}, 22(6): 1043--1083.

\bibitem[{Gebser et~al.(2019)Gebser, Kaminski, Kaufmann, and
  Schaub}]{gekakasc17a}
Gebser, M.; Kaminski, R.; Kaufmann, B.; and Schaub, T. 2019.
\newblock Multi-shot {ASP} solving with clingo.
\newblock \emph{Theory and Practice of Logic Programming}, 19(1): 27--82.

\bibitem[{Gebser et~al.(2013)Gebser, Kaufmann, Otero, Romero, Schaub, and
  Wanko}]{gekaotroscwa13a}
Gebser, M.; Kaufmann, B.; Otero, R.; Romero, J.; Schaub, T.; and Wanko, P.
  2013.
\newblock Domain-specific Heuristics in Answer Set Programming.
\newblock In {desJardins}, M.; and Littman, M., eds., \emph{Proceedings of the
  Twenty-seventh National Conference on Artificial Intelligence (AAAI'13)},
  350--356. {AAAI} Press.

\bibitem[{Gelfond(1991)}]{gelfond91a}
Gelfond, M. 1991.
\newblock Strong Introspection.
\newblock In Dean, T.; and McKeown, K., eds., \emph{Proceedings of the Nineth
  National Conference on Artificial Intelligence (AAAI'91)}, 386--391. {AAAI}
  Press.

\bibitem[{Gelfond(1994)}]{gelfond94a}
Gelfond, M. 1994.
\newblock Logic Programming and Reasoning with Incomplete Information.
\newblock \emph{Annals of Mathematics and Artificial Intelligence}, 12(1-2):
  89--116.

\bibitem[{Gelfond(2011)}]{gelfond11a}
Gelfond, M. 2011.
\newblock New Semantics for Epistemic Specifications.
\newblock In Delgrande, J.; and Faber, W., eds., \emph{Proceedings of the
  Eleventh International Conference on Logic Programming and Nonmonotonic
  Reasoning (LPNMR'11)}, volume 6645 of \emph{Lecture Notes in Artificial
  Intelligence}, 260--265. Springer-Verlag.

\bibitem[{Gelfond and Lifschitz(1988)}]{gellif88b}
Gelfond, M.; and Lifschitz, V. 1988.
\newblock The Stable Model Semantics for Logic Programming.
\newblock In Kowalski, R.; and Bowen, K., eds., \emph{Proceedings of the Fifth
  International Conference and Symposium of Logic Programming (ICLP'88)},
  1070--1080. MIT Press.

\bibitem[{Gelfond and Lifschitz(1991)}]{gellif91a}
Gelfond, M.; and Lifschitz, V. 1991.
\newblock Classical Negation in Logic Programs and Disjunctive Databases.
\newblock \emph{New Generation Computing}, 9: 365--385.

\bibitem[{Gelfond and Przymusinska(1993)}]{gelpri93a}
Gelfond, M.; and Przymusinska, H. 1993.
\newblock Reasoning on Open Domains.
\newblock In Pereira, L.; and Nerode, A., eds., \emph{Proceedings of the Second
  International Conference on Logic Programming and Nonmonotonic Reasoning
  (LPNMR'93)}, volume 928 of \emph{Lecture Notes in Artificial Intelligence},
  397--413. Springer-Verlag.

\bibitem[{Kahl, Leclerc, and Son(2016)}]{kaleso16a}
Kahl, P.; Leclerc, A.; and Son, T. 2016.
\newblock A Parallel Memory-efficient Epistemic Logic Program Solver: Harder,
  Better, Faster.
\newblock In Bogaerts, B.; and Harrison, A., eds., \emph{Proceedings of the
  Ninth Workshop on Answer Set Programming and Other Computing Paradigms
  (ASPOCP'16)}.

\bibitem[{Kahl et~al.(2015)Kahl, Watson, Balai, Gelfond, and
  Zhang}]{kawabagezh15a}
Kahl, P.; Watson, R.; Balai, E.; Gelfond, M.; and Zhang, Y. 2015.
\newblock The language of epistemic specifications (refined) including a
  prototype solver.
\newblock \emph{Journal of Logic and Computation}.

\bibitem[{Kahl et~al.(2020)Kahl, Watson, Balai, Gelfond, and
  Zhang}]{kawabagezh20a}
Kahl, P.; Watson, R.; Balai, E.; Gelfond, M.; and Zhang, Y. 2020.
\newblock The language of epistemic specifications (refined) including a
  prototype solver.
\newblock \emph{Journal of Logic and Computation}, 30(4): 953--989.

\bibitem[{Kaminski et~al.(2023)Kaminski, Romero, Schaub, and
  Wanko}]{karoscwa21a}
Kaminski, R.; Romero, J.; Schaub, T.; and Wanko, P. 2023.
\newblock How to Build Your Own {ASP}-based System?!
\newblock \emph{Theory and Practice of Logic Programming}, 23(1): 299--361.

\bibitem[{Leclerc and Kahl(2018)}]{leckah18a}
Leclerc, A.; and Kahl, P. 2018.
\newblock A survey of advances in epistemic logic program solvers.
\newblock arXiv:1809.07141.

\bibitem[{Lifschitz(2008)}]{lifschitz08b}
Lifschitz, V. 2008.
\newblock What Is Answer Set Programming?
\newblock In Fox, D.; and Gomes, C., eds., \emph{Proceedings of the
  Twenty-third National Conference on Artificial Intelligence (AAAI'08)},
  1594--1597. {AAAI} Press.

\bibitem[{Lifschitz(2019)}]{lifschitz19a}
Lifschitz, V. 2019.
\newblock \emph{Answer Set Programming}.
\newblock Springer-Verlag.

\bibitem[{Lifschitz and Turner(1994)}]{liftur94a}
Lifschitz, V.; and Turner, H. 1994.
\newblock Splitting a logic program.
\newblock In \emph{Proceedings of the Eleventh International Conference on
  Logic Programming}, 23--37. MIT Press.

\bibitem[{Schaub and Woltran(2018)}]{schwol18a}
Schaub, T.; and Woltran, S. 2018.
\newblock Special Issue on Answer Set Programming.
\newblock \emph{K{\"u}nstliche Intelligenz}, 32(2-3): 101--103.

\bibitem[{Shen and Eiter(2016)}]{sheeit16a}
Shen, Y.; and Eiter, T. 2016.
\newblock Evaluating Epistemic Negation in Answer Set Programming.
\newblock \emph{Artificial Intelligence}, 237: 115--135.

\bibitem[{Shen and Eiter(2017)}]{sheeit17a}
Shen, Y.; and Eiter, T. 2017.
\newblock Evaluating Epistemic Negation in Answer Set Programming (Extended
  Abstract).
\newblock In Sierra, C., ed., \emph{Proceedings of the Twenty-sixth
  International Joint Conference on Artificial Intelligence (IJCAI'17)},
  5060--5064. IJCAI/AAAI Press.

\bibitem[{Son et~al.(2017)Son, Le, Kahl, and Leclerc}]{solekale17a}
Son, T.; Le, T.; Kahl, P.; and Leclerc, A. 2017.
\newblock On Computing World Views of Epistemic Logic Programs.
\newblock 1269--1275.

\bibitem[{Truszczynski(2011)}]{truszczynski11b}
Truszczynski, M. 2011.
\newblock Revisiting Epistemic Specifications.
\newblock In  \cite{mg65}, 315--333.

\end{thebibliography}

\section*{Proof of Results}\label{sec:proofs}

\begin{lemma}\label{lem:aux1:thm:normal.form}
    Let~$\Pi$ be a program and~$\wv$ be a belief interpretation of~$\NF{\Pi}$.
    If~$\wv'$ is a belief interpretation such that
    \begin{itemize}
        \item $\wv \models \K \nott{a}$ iff~$\wv \models \K \neg a$ iff~$\wv' \models \K \neg a$, and
        \item $\wv \models \K \notnot{a}$ iff~$\wv \models \K \neg\neg a$ iff~$\wv' \models \K \neg\neg a$, and
    \end{itemize}
    then,
    \begin{align}
        \SM{\NF{\Pi}^\wv} = \big\{ \ M \cup \hat{M} \mid M \in \SM{\Pi^{\wv'}} \ \big\}
        \label{eq:1:lem:aux1:thm:normal.form}
    \end{align}
    and~$\hat{M} = \{ \nott{a} \mid a \in M \} \cup \{ \notnot{a} \mid a \in M \}$.
\end{lemma}

\begin{proof}
Then, $\NF{\Pi}^{\wv} = \Pi^{\wv'} \cup \Gamma$ where~$\Gamma$ is the set of rules of the form~$\nott{a} \leftarrow \neg a$ and~$\notnot{a} \leftarrow \neg\neg a$ in~$\NF{\Pi}$.
We use the Splitting Theorem~\cite{liftur94a} with splitting set~${U = \At{\Gamma}\setminus\At{\Pi}}$.
This set splits the program~$\NF{\Pi}^{\wv}$ into a bottom program~$\Pi^{\wv'}$ and a top program~$\Gamma$.
Note that atoms~$\nott{a}$ and~$\notnot{a}$ do not occur in~$\Pi^{\wv'}$ because they only occur in the subjective literals, which are removed by the reduct.
Let~$e_U(\Gamma,M)$ be the result of replacing in~$\Gamma$ each atom~$a \in \At{\Pi}$ by~$\top$ if~$a \in M$ and by~$\bot$ otherwise.
Then, $e_U(\Gamma,M)$ is the set of facts~$\hat{M}$ and, by the Splitting Theorem, \eqref{eq:1:lem:aux1:thm:normal.form} follows.
\end{proof}

\begin{proof}[Proof of Theorem~\ref{thm:normal.form}]
Pick a worldview~$\wv$ of program~$\NF{\Pi}$ and let~$\wv' = \restr{\wv'}{\At{\Pi}}$.
Then, $\wv = \SM{\NF{\Pi}^{\wv}}$ and, thus, every~$M \in \wv$ and every atom~$a$ satisfy $\nott{a} \in M$ iff~$a \notin M$ and $\notnot{a} \in M$ iff~$a \in M$.
Hence, the preconditions of Lemma~\ref{lem:aux1:thm:normal.form} are satisfied and, thus,
\begin{align*}
    \wv' &= \restr{\wv}{\At{\Pi}} = \restr{\SM{\NF{\Pi}^{\wv}}}{\At{\Pi}}
    \\
    &= \restr{\{ M \cup \hat{M} \mid M \in \SM{\Pi^{\wv'}} \}}{\At{\Pi}}
    \\
    &= \{ M \mid M \in \SM{\Pi^{\wv'}} \}
\end{align*}
This implies that~$\wv'$ is a worldview of~$\Pi$.
Conversely, let~$\wv'$ be a worldview of~$\Pi$ and let~$\wv = \{ M \cup \hat{M} \mid M \in \wv' \}$.
Then, by Lemma~\ref{lem:aux1:thm:normal.form} again, we get
\begin{align*}
    \SM{\NF{\Pi}^{\wv}} &= \{ M \cup \hat{M} \mid M \in \SM{\Pi^{\wv'}} \}
    \\
    &= \{ M \cup \hat{M} \mid M \in \wv' \}
    \\
    &= \wv
\end{align*}
This implies that~$\wv$ is a worldview of~$\NF{\Pi}$.
\end{proof}

\begin{proof}[Proof of Proposition~\ref{prop:basic.tester.program}]
    Let~$\mathit{Choices(\Pi)}$ be the program containing choice rule~$\{ \kk{a} \}$ for each atom~$\kk{a} \in K(\Pi)$ and let $\Pi_0$ be the objective program
    \begin{align*}
        k(\Pi) & \cup \{ \neg\neg a \mid a \in k(\wv) \} \cup \{ \neg a  \mid \neg a \in k(\wv) \}
        \\&\cup \mathit{Choices(\Pi)}.
    \end{align*}
    By definition, it follows that $\SM{T_0(\Pi); k(\wv)} = \SM{\Pi_0}$.
    Furthermore, program~$\Pi_0$ has the same stable models as program~$\Pi_1 =\Pi_2 \cup \mathit{Facts(\Pi,\wv)}$ where~$\mathit{Facts(\Pi,\wv)} = K(\Pi) \cap k(\wv)$ is the a set of facts for the auxiliary atoms that correspond to true epistemic atoms in~$\wv$ and~$\Pi_2$ is obtained from~$k(\Pi)$ by removing all rules with an atom~$\kk{a}$ with~$\kk{a} \in \At{\Pi} \setminus k(\wv)$ in the body.
    Note that these atoms cannot be true in~$\Pi_0$ because it contains constraints~$\{ \neg a  \mid \neg a \in k(\wv) \}$ nor in~$\Pi_1$ because they do not occur in the head of any of its rules.
    Moreover, let~$\Pi_3$ be the result of removing all atoms of the form~$\kk{a}$ from~$\Pi_2$.
    Then~$\Pi_3 = \Pi^\wv$ and~$\SM{\Pi_1} =  \{ M \cup k(\wv) \mid M \in \SM{\Pi_3} \}$.
    Hence, \eqref{eq:1:prop:basic.generator.tester.program} holds.
\end{proof}

\begin{proof}[Proof of Theorem~\ref{thm:guess-and-check}]
Assume that~$\wv$ is a worldview of~$\Pi$.
By definition of worldview and tester program, it follows that 
$$\wv \ = \ \SM{\Pi^\wv} \ = \ \restr{\SM{T(\Pi); k(\wv)}}{\At{\Pi}}$$
From this and the definition of generator program,
it follows that there is a stable model~$M$ of~$G(\Pi)$ satisfying~$k(\wv) = k(M)$.
Futhermore, by definition of~$k(\wv)$ and the tester program, it follows that
\begin{gather*}
    \begin{aligned}
        k(\wv) \ &= \ \{ ka \in K(\Pi) \mid \wv \models \K a \}
        \\
        &= \ \{ ka \in K(\Pi) \mid \restr{\SM{T(\Pi); k(\wv)}}{\At{\Pi}} \models \K a \}
        \\
        &= \ \{ ka \in K(\Pi) \mid a \in \SMC{T(\Pi); k(\wv)} \cap \At{\Pi} \}
        \\
        &= \ \{ ka \in K(\Pi) \mid a \in \SMC{T(\Pi); k(\wv)} \}
        \\
        &= \ \{ ka \in K(\Pi) \mid a \in \SMC{T(\Pi); k(M)} \}
    \end{aligned}
    \label{eq:2:thm:guess-and-check}
\end{gather*}
Hence, there is a stable model~$M$ of~$G(\Pi)$ with~$k(\wv) = k(M)$ satisfying~\eqref{eq:1:thm:guess-and-check}.
\\[5pt]
Assume now that there is a stable model~$M$ of~$G(\Pi)$ satisfying~\eqref{eq:1:thm:guess-and-check} 
and take the set of stable models~$\wv = \restr{\SM{T(\Pi); k(M)}}{\At{\Pi}}$.
Similar to the above, it follows that
\begin{gather*}
    k(\wv) \quad = \quad \{ ka\in K(\Pi) \mid a \in \SMC{T(\Pi); k(M)} \}.
\end{gather*}
and, thus, $k(\wv) = k(M)$ and~$\wv = \restr{\SM{T(\Pi); k(\wv)}}{\At{\Pi}}$.
By definition of tester program, it follows that~$\wv$ is a worldview of~$\Pi$.
\end{proof}

\begin{proof}[Proof of Proposition~\ref{prop:most.basic.generator.program}]
    Let~$\wv$ be a worldview of~$\Pi$, that is~$\wv = \SM{\Pi^\wv}$.
    Then,
    $$
    \{ M \cup k(\wv) \mid M \in \wv \} \ = \ \{ M \cup k(\wv) \mid M \in \SM{\Pi^\wv} \}
    $$
    and, by Proposition~\ref{prop:basic.tester.program}, it follows that
    $$\{ M \cup k(\wv) \mid M \in \wv \} \ = \ \SM{T_0(\Pi); k(\wv)}.$$
    Since~$\wv$ is a worldview, it is non\nobreakdash-empty and, thus, there is~$M' \in \SM{T_0(\Pi); k(\wv)}$
    such that~$M' = M \cup k(\wv)$ with~$M \in \wv$.
    Hence, $k(\wv) = k(M')$.
    This means that~$T_0$ satisfies the first condition of Definition~\ref{def:generator.program}.
    
    Pick now a stable model~$M$ of~$T_0(\Pi)$ and let~$\wv$ be any set of interpretations such that~$k(\wv) = k(M)$.
    Then, 
    $$M \in \SM{T_0(\Pi); k(M')} = \SM{T_0(\Pi); k(\wv)}.$$
    By Proposition~\ref{prop:basic.tester.program}, it follows that~$(M \cap \At{\Pi}) \in \SM{\Pi^\wv}$.
    Hence, the second condition of Definition~\ref{def:generator.program} is  satisfied.%
\end{proof}

\begin{proof}[Proof of Proposition~\ref{prop:basic.generator.program}]
    Let~$\wv$ be a worldview of~$\Pi$, that is~$\wv = \SM{\Pi^\wv}$.
    By Proposition~\ref{prop:most.basic.generator.program}, it follows that~$T_0(\Pi)$ is a generator program for~$\Pi$ and, thus, there is a stable model~$M$ of~$T_0(\Pi)$ satisfying~$k(\wv) = k(M)$ and~$\restr{M}{\At{\Pi}} \in \wv$.
    Pick any atom~${ka \in M}$.
    Then, $ka \in k(M) = k(\wv)$ and, thus, $\wv \models \K a$.
    Since~$\restr{M}{\At{\Pi}} \in \wv$, this implies that~$a \in M$ and, thus, it satisfies~\eqref{eq:consistency.constraints}.
    Hence, $M$ is a stable model of~$G_0(\Pi)$.
    Hence, the first condition of Definition~\ref{def:generator.program} is satisfied.
    
    Pick now a stable model~$M$ of~$G_0(\Pi)$ and let~$\wv$ be any set of interpretations such that~$k(\wv) = k(M)$.
    Then, $M$ is also a stable model of~$T_0(\Pi)$ and, since~$T_0(\Pi)$ is a generator program for~$\Pi$, it follows that~$\restr{M}{\At{\Pi}} \in \SM{\Pi^\wv}$.
    Hence, the second condition of Definition~\ref{def:generator.program} is also satisfied.
\end{proof}

\begin{proof}[Proof of Lemma~\ref{lem:propagation.generator.program}]
    We use the Splitting Theorem~\cite{liftur94a} with splitting set~${U = \At{G_0(\Pi)}}$.
    This set splits the program into a bottom program~$G_0(\Pi)$ and a top program~$\kp{(\Pi)}$.
    Let~$e_U(\kp{(\Pi)},M)$ be the result of removing from~$\kk{(\Pi)}$ 
    each rule containing a literal~$L$ with an atom of the form~$\kk{a}$ such that~$M \not\models L$; and removing all literals with an atom of the form~$\kk{a}$ from the remaining rules.    %
    By the Splitting Theorem, $M'$ is a stable model of~$G_0(\Pi) \cup \kp{(\Pi)}$ iff there is a stable model~$M$ of~$G_0(\Pi)$ and a stable model~$M^p \cup M^n$ of~$e_U(\kp{(\Pi)},M)$ such that~$M' = M \cup M^p \cup M^n$.
    Note that~$e_U(\kp{(\Pi)},M)$ is a Horn program and, thus, it has a unique stable model.
    It only remains to be shown that~\eqref{eq:1:prop:propagation.generator.program} and~\eqref{eq:2:prop:propagation.generator.program} hold.
    Since~${M^p \cup M^n}$ is a stable mode of~$e_U(\kp{(\Pi)},M)$ and the latter is a Horn program, it follows that~$M^p \cup M^n = T_P^\kappa$ for some integer~$\kappa \geq 0$ where~$T_P$ is the direct consequences operator of~$e_U(\kp{(\Pi)},M)$.
    We prove
    \begin{align*}
        T_P^\kappa &\subseteq 
        \{ \kp{a} \mid \SM{\Pi^\wv} \models \K a \} \cup 
        \{ \kpnott{a} \mid \SM{\Pi^\wv} \models \K \neg a \} 
        \\
        &\ \cup 
        \{ \kpnott{r_i} \mid \SM{\Pi^\wv} \models \K \neg L \text{ for some } L \in \mathit{body(r_i)} \} 
    \end{align*}
    for every~$\kappa \geq 0$.
    The proof is by induction on~$\kappa$.
    The case that~$\kappa = 0$ is trivial because~$T_P^\kappa = \emptyset$.
    Assume now that the statement holds for~$\kappa$ and pick~$\kp{a} \in T_P^{\kappa+1}$.
    Then, there is a rule~$r_j \in \Pi$ of the form of~\eqref{eq:propagation.rule.regular} such that
    \begin{itemize}
        \item $\{\kp{a_2},\dotsc,\kp{a_k}\} \subseteq T_P^\kappa$,
        \item $\{\kpnott{a_{k+1}},\dotsc,\kpnott{a_m}\} \subseteq T_P^\kappa$,
        \item $M \models L_\lambda$ for all~$m+1 \leq \lambda \leq n$.
    \end{itemize}
    By induction hypothesis and the fact that~$\kk{M} = \kk{\wv}$, this implies
    \begin{itemize}
        \item $\SM{\Pi^\wv} \models \K a_\lambda$ for all~$2 \leq \lambda \leq k$,
        \item $\SM{\Pi^\wv} \models \K \neg a_\lambda$ for all~$k+1 \leq \lambda \leq m$,
        \item $\wv \models L_\lambda$ for all~$m+1 \leq \lambda \leq n$.
    \end{itemize}
    Hence, ${\SM{\Pi^\wv} \models \K a}$ follows.
    The two remaining cases---${\kpnott{r_j} \in T_P^{\kappa+1}}$ and~${\kpnott{a} \in T_P^{\kappa+1}}$---are similar.
\end{proof}

\begin{proof}[Proof of Proposition~\ref{prop:propagation.generator.program} ]
    By Lemma~\ref{lem:propagation.generator.program}, 
    it follows that~$|\SM{G_0(\Pi)\cup \kp{(\Pi)}}| = |\SM{G_0(\Pi)}|$.
    Since~$G_1(\Pi)$ is the result of adding some integrity constrains to~$G_0(\Pi) \cup \kp{(\Pi)}$, the first property holds.
    For the second property, we can compute a stable model~$M$ of~$G_0(\Pi)\cup kp{(\Pi)}$ from a stable model~$M'$ of~$G_0(\Pi)$ by adding the consequences~$e_U(\kp{(\Pi)},M')$.
    Since this is a definite program, its consequence can be computed in linear time.
    Checking that the constraints in~$G_1(\Pi) \setminus G_1(\Pi)\cup kp{(\Pi)}$ are satisfied also can be computed in linear time.
    Finally, for the third property, consider the program~$\Pi_i$ of Example~\ref{ex:propagation}.
\end{proof}

\end{document}